\documentclass[letterpaper, 10 pt, conference]{ieeeconf}  

\IEEEoverridecommandlockouts                              
\usepackage{graphicx}

\usepackage{amsmath}
\usepackage{amssymb}
\usepackage{amsfonts}
\usepackage{mathtools}
\usepackage{MnSymbol}

\usepackage{subcaption}
\usepackage{amsthm}
\usepackage[usenames,dvipsnames,svgnames,table]{xcolor}
\usepackage{float}
\floatstyle{plaintop}
\restylefloat{table}
\usepackage{cite}
\usepackage[normalem]{ulem}
\usepackage{tikz}

\usepackage[vlined,ruled,linesnumbered]{algorithm2e}
\theoremstyle{plain}
\newtheorem{theorem}{Theorem}
\newtheorem{definition}{Definition}

\newtheorem{problem}{Problem}
\newtheorem{proposition}{Proposition}

\newcommand{\levent}{\lozenge}
\newcommand{\lalways}{\square}
\newcommand{\rsep}{\mathrel{{.}{.}}\nobreak}
\newcommand{\range}[2]{[{#1 {\rsep} #2}]}

\newcommand{\Dataset}{{\mathcal{S}}}

\newcommand{\Classes}{{\mathcal{C}}}

\providecommand{\norm}[1]{{\lVert#1\rVert}}
\newcommand{\pa}[1]{{\mathrm{pa}(#1)}}
\newcommand{\ch}[1]{{\mathrm{ch}(#1)}}

\makeatletter
\let\NAT@parse\undefined
\makeatother
\RequirePackage[bookmarks,unicode,colorlinks=true, allcolors=blue]{hyperref}%
\title{\LARGE \bf
Learning Optimal Signal Temporal Logic Decision Trees for Classification:
A Max-Flow MILP Formulation
}
  
\author{Kaier Liang, Gustavo A. Cardona, Disha Kamale, and Cristian-Ioan Vasile
\thanks{Kaier Liang, Gustavo A. Cardona,  Disha Kamale, and Cristian-Ioan Vasile are with the Mechanical Engineering and Mechanics Department at Lehigh University, PA, USA: {\tt\small \{kal221, gcardona, ddk320, cvasile\}@lehigh.edu}}        
}   
\begin{document}

\maketitle
\thispagestyle{empty}
\pagestyle{empty}

\begin{abstract}
This paper presents a novel framework for inferring timed temporal logic properties from data. 
The dataset comprises pairs of finite-time system traces and corresponding labels, denoting whether the traces demonstrate specific desired behaviors, e.g. whether
the ship follows a safe route or not. 
Our proposed approach leverages decision-tree-based methods to infer Signal Temporal Logic classifiers using primitive formulae.
We formulate the inference process as a mixed integer linear programming optimization problem, recursively generating constraints to determine both data classification and tree structure. 
Applying a max-flow algorithm on the resultant tree transforms the problem into a global optimization challenge, leading to improved classification rates compared to prior methodologies.
Moreover, we introduce a technique to reduce the number of constraints by exploiting the symmetry inherent in STL primitives, which enhances the algorithm's time performance and interpretability.
To assess our algorithm's effectiveness and classification performance, we conduct three case studies involving two-class, multi-class, and complex formula classification scenarios. 
\end{abstract}

\section{Introduction}
\label{sec:intro}
The ever-increasing complexity of modern systems has resulted in an urgent need for sophisticated techniques that can help understand and classify temporal behaviors from time-series data. 
Machine learning (ML)~\cite{kotsiantis2007supervised, soofi2017classification} approaches have emerged as effective tools for this task, particularly in two-class classification, where the objective is to differentiate between desired and undesired system behaviors. 
However, the opaque nature of traditional ML algorithms often hinders interpretability and insight into system dynamics~\cite{krishnan2020against, gilpin2018explaining}.

To address this limitation, formal methods such as Signal Temporal Logic (STL)~\cite{maler2004, donze2013signal} have been gaining traction as a means of specifying temporal properties of real-valued signals. 
STL provides a structured language for expressing complex temporal and logical behaviors, offering both readability and interoperability. 
By formulating temporal properties as logical formulae, STL facilitates the inference of system behaviors from labeled time-series data~\cite{mohammadinejad2020interpretable}.

While early methods in this domain focused on manual parameter synthesis from predefined template formulae~\cite{asarin2012parametric, jin2013mining, hoxha2018mining, bakhirkin2018efficient}, recent advancements have sought to automate this process by inferring both the structure and parameters of STL formulae directly from data~\cite{kong2016temporal}. 
The authors of~\cite{jones2014anomaly} introduced a fragment of STL called inference parametric signal temporal logic (iPSTL), which enables the classification problem to be formulated as an optimization problem. 
However, this approach faces challenges such as high computational cost due to nonlinear parameter optimization routines and constructing a directed acyclic graph (DAG) based on the ordering of PSTL formulae, which may not necessarily improve classification performance.
Other recent works, such as ~\cite{bartocci2014data, bufo2014temporal}, have addressed the two-class classification problem by building generative models for each class and deriving a discriminative formula that maximizes the probability of satisfaction for one model while minimizing it for the other. 
Despite the potential of these methods, they require building models of the system under analysis, which entails domain expertise and substantial amounts of data.
On the other hand, decision tree-based frameworks have emerged as promising approaches for efficiently learning STL formulae. 
These frameworks construct trees where each node encapsulates simple formulae optimized from a predefined set of primitives~\cite{aasi2022classification, bombara2016decision}. 
Lastly, the authors in ~\cite{li2023learning} use the different layers of neural networks to learn different aspects of STL formulae that allow a compact and efficient binary interpretable classification for level-one STL formulae, and multi-class inference~\cite{linard2022inference,li2024multi}.

Despite recent advances, the formulae generated by existing algorithms often suffer from verbosity and complexity, limiting their interpretability and practicality in real-world scenarios. 
In response to these challenges, this paper introduces an innovative approach: optimal max-flow tree data classification using Signal Temporal Logic (STL) inference.
Our method builds upon the concept of tree-based structures, similar to previous works such as ~\cite{aasi2022classification, bombara2016decision}, to encode STL primitives recursively and facilitate the classification of labeled datasets. 
However, unlike previous approaches, we leverage this tree structure to formulate a max-flow optimization problem~\cite{aghaei2021strong}, enabling us to simultaneously determine the tree structure and the inferred STL specification.
At each node of the tree, we encapsulate STL primitives and make branching decisions based on the classification of the input data. 
By formulating the classification task as a Mixed Integer Linear Programming (MILP) problem, our approach offers several advantages over previous methods. 
Unlike~\cite{aasi2022classification, bombara2016decision}, our approach enables global optimization and can be adapted for multi-class classification scenarios. 
Moreover, we reduce primitive redundancy, resulting in fewer constraints, which improves time performance and reduces complexity.
In contrast to~\cite{li2023learning}, our method considers a broader set of STL primitives and avoids the pitfalls associated with nonlinear optimization, such as getting trapped in local minima. 
By simultaneously maximizing classification accuracy and minimizing formula complexity, our approach ensures optimal performance with enhanced interpretability.
We validate the efficacy of our method through multiple case studies, including two-class and multi-class classification scenarios. 
The results demonstrate significantly improved classification performance and interpretability compared to existing approaches, thereby paving the way for more effective temporal behavior classification and analysis.

The main contributions of this work are, 
\begin{enumerate}
    \item Proposing a novel decision tree-based STL inference algorithm that runs a global optimization and results in an improved classification rate as compared to the existing related approaches.
    \item Formulating a Mixed Integer Linear Programming encoding that classifies data and generates the rules for growing the tree structure based on STL primitives followed by a global max-flow optimization approach that guarantees high-performance classification.
    \item Generating a reduction of constraints based on the symmetry of STL primitives, resulting in more efficient and improved time performance.
    \item Demonstrating three case studies to show the algorithm's capability to handle two-class, multi-class, and complex formulae classification and compare its performance with other approaches.
\end{enumerate}



\section{Preliminaries and Notation}
\label{sec:preliminaries}
Let $\mathbb{R}$ denote the set of all real numbers, $\mathbb{Z}$ the set of integers, $\mathbb{B}$ the binary set, and $\mathbb{Z}_{\geq 0}$ the set of non-negative integers. 
For a set $\mathcal{S}$, $2^\mathcal{S}$ and $|\mathcal{S}|$ represent its power set and cardinality.
We have $\alpha+S = \{\alpha+x\mid x\in S\}$.
The integer interval (range) from $a$ to $b$ is $\range{a}{b}$. We use $\underline{I}=a$ and $\bar{I}=b$.
A discrete-time signal $s$, with time horizon $H \in \mathbb{Z}_{\geq 0}$, is defined as a function $s: \range{0}{H} \to \mathbb{R}^d$ mapping each time-step to an $d$-dimensional vector of real values.
The $j$-th component of $x$ is given by $x_j$, $j\in \range{1}{d}$.

\subsection{Signal Temporal Logic}
\label{sec:stl}
Consider a discrete-time signal $s: \range{0}{H} \to \mathbb{M}$ with values in the compact space $\mathbb{M} \subseteq \mathbb{R}^d$. 
Introduced in \cite{maler2004}, Signal Temporal Logic (STL) is a specification language that expresses real-time properties. Its syntax is defined as follows
\begin{equation}
\label{eq:stl-syntax}
    \phi ::= \top \mid h(s) \geq \pi \mid \lnot \phi \mid \phi_1 \land \phi_2 \mid \phi_1 \lor \phi_2 \mid  \levent_{I} \phi \mid \lalways_{I} \phi,
\end{equation}
where $\phi$, $\phi_1$, and $\phi_2$ are STL formulae,
$\top$ is the logical \emph{True} value,
$h(s) \geq \pi$ is a predicate with $h: \mathbb{R}^d \to \mathbb{R}$ and threshold value $\pi \in \mathbb{R}$,
$\lnot$, $\land$, and $\lor$ are the Boolean negation, conjunction, and disjunction operators.
Discrete-time temporal operators \emph{eventually} $\levent_I$, and \emph{always} $\lalways_I$ with $I = \range{\underline{I}}{\bar{I}}$
a discrete-time interval, $\bar{I} \geq \underline{I} \geq 0$,
are defined in the usual way~\cite{maler2004}.
Linear predicates of the form $s \sim \pi$, $\sim\, \in \{>, \leq, <\}$, follow via
negation and sign change. The logical \emph{False} value is $\bot = \lnot \top$.

The (qualitative) semantics of STL formulae over signals $s$ at time $k$ is recursively defined in~\cite{maler2004} as
\begin{equation}
\label{eq:stl-semantics}
  \begin{aligned}
    (s, k) \models (h(s) \geq \pi) \equiv{} & h(s(k)) \geq \pi,\\
    (s, k) \models \lnot \phi \equiv{} & (s, k) \nmodels \phi,\\
    (s, k) \models \phi_1 \land \phi_2 \equiv{} & \big( (s, k)\models \phi_1 \big) \land \big( (s, k)\models \phi_2 \big),\\
    (s, k) \models \phi_1 \lor \phi_2 \equiv{} & \big( (s, k)\models \phi_1 \big) \lor \big( (s, k)\models \phi_2 \big),\\
    (s, k) \models \levent_I \phi \equiv{} & \exists k' \in k + I \textrm{ s.t. } (s, k') \models \phi, \\
    (s, k) \models \lalways_I \phi \equiv{} & \forall k' \in k + I \textrm{ s.t. } (s, k') \models \phi,
  \end{aligned}  
\end{equation}
where $\models$ and $\nmodels$ denote satisfaction and violation, respectively.
A signal $s$ satisfying $\phi$, denoted as $s\models \phi$, is true if $(s, 0) \models \phi$.

In addition to Boolean semantics, STL admits quantitative semantics,
called \emph{robustness}, that indicates how much a signal satisfies
or violates a specification~\cite{fainekos2009robustness,donze2010}.
The robustness score $\rho(s, \phi, k)$ is recursively defined as
\begin{equation}
\label{eq:stl-robustness}
\begin{aligned}
    \rho(s, \top, k) & {} = \rho_\top,\\
    \rho(s, h(s) \geq \pi, k) & {}= h(s(k)) - \pi,\\
    \rho(s, \lnot \phi, k) & {} = -\rho(s, \phi, k),\\
    \rho(s, \phi_1 \land \phi_2, k) & {}= \min(\rho(s, \phi_1, k), \rho(s, \phi_2, k)),\\
    \rho(s, \phi_1 \lor \phi_2, k) & {}= \max(\rho(s, \phi_1, k), \rho(s, \phi_2, k)),\\
    \rho(s,\lalways_I \phi, k) & {}= \min_{k' \in k+I} \rho(s, \phi, k'),\\
    \rho(s,\levent_I \phi, k) & {}= \max_{k' \in k+I} \rho(s, \phi, k'),\\
\end{aligned}
\end{equation}
where $\rho_\top = \sup_{s, \pi}\{|h(s) - \pi|\}$ is the maximum robustness.

\begin{theorem}[Soundness~\cite{donze2010}]
\label{thm:stl-robustness-soundness}
Let $s$ be a signal and $\phi$ an STL formula.
It holds $\rho(s, \phi, k) > 0 \Rightarrow (s, k) \models \phi$ for satisfaction
and $\rho(s, \phi, k) < 0 \Rightarrow (s, k) \nmodels \phi$ for violation.
\end{theorem}

The time horizon of an STL formula~\cite{Dokhanchi2014} is defined as
\begin{equation*}\small
\norm{\phi} =
\begin{cases}
0, & \mbox{if } \phi \in \{\top, h(s) \geq \pi \}, \\
\norm{\phi_1}, & \mbox{if } \phi = \lnot \phi_1,\\
\max\{\norm{\phi_1}, \norm{\phi_2}\}, & \mbox{if } \phi \in \{\phi_1 \land \phi_2, \phi_1 \lor \phi_2 \},\\
\bar{I} + \norm{\phi_1}, & \mbox{if } \phi \in \{\levent_I \phi_1, \lalways_{I} \phi_1 \}.
\end{cases}
\end{equation*}

\section{Problem Formulation}
\label{sec: problem}
This section introduces the data classification problem by inferring STL specifications.
Let us consider a labeled dataset $\Dataset:= \{(s^i, \ell^i)\}_{i \in \mathcal{I}}$, where $s^i$ and $\ell^i$ represent the $i$-th signal and its corresponding label. 
We denote the set of all possible classification classes as $\Classes=\range{1}{|\Classes|}$, where $\ell^i=c$ signifies that the $i$-th sample is labeled with class $c \in \Classes$. 
Our aim is to infer STL formulae $\phi_c$ that encapsulate properties inherent to each data class $c \in \Classes$. 
The formulae $\phi_c$ enable the classification of signals not present in the dataset $\Dataset$.
Formally, we consider the following multi-class classification problem.
 %
%
\begin{problem}
\label{pb: p1}
Given a labeled data set $\Dataset$, find a set of mutually-exclusive STL formulae $\Phi = \{\phi_c\}_{c\in \Classes}$ that maximizes the correct classification rate $CCR(\Phi)$ where 
\begin{equation}
\label{eq:CCR}
CCR(\Phi)\colon=\frac{\sum_{c \in \Classes}\left|\{s^i \mid s^i \models \phi_c \wedge \ell^i=c\}\right|}{|\mathcal{I}|}.
\end{equation}
such that $\phi_c \land \phi_{c'} \equiv \bot$ for all $c, c' \in \Classes$ with $c\neq c'$.
\end{problem}
    
\section{Solution}
\label{sec: MILP}
In this section, we propose an approach to learning Signal Temporal Logic formulae in the form of decision trees. 
The tree branches from the root to the leaves correspond to the classification process, where leaves are associated with classes and intermediate nodes with decisions in the form of primitive STL formulae.
We infer the structure of the decision tree, the primitives to use for each decision node and their spatial and temporal parameters, and the classes associated with the leaf nodes such that the correct classification rate is maximized and the number of decision nodes is minimized.
We cast the temporal inference problem as a Mixed Integer Linear Programming (MILP) problem using a max flow encoding~\cite{aghaei2021strong}.


Our classification method employs a tree-based structure~\cite{aghaei2021strong}, akin to a binary decision tree but with a unique feature: it incorporates a source node $\mathbf{s}$ and a sink node $\mathbf{t}$, as illustrated in Fig.~\ref{fig: tree}. 
Furthermore, each node is connected to the sink node. 
We refer to this configuration as a \emph{classification tree}, defined as follows.
\begin{definition}[Classification Tree]
    \label{def:DAG}
    A classification tree $\mathcal{T}=(\mathcal{N},\mathcal{L}, \mathcal{E}, \mathbf{s}, \mathbf{t})$ is a Directed Acyclic Graph (DAG) with single source $\mathbf{s}$ and single sink $\mathbf{t}$. 
    Nodes in the final layer of the tree are referred to as leaf nodes, denoted by $\mathcal{L}$, while nodes situated between the source and the leaf nodes are internal nodes, denoted by $\mathcal{N}$ and $\mathcal{E}\subseteq \left(\mathcal{N} \cup \{\mathbf{s}\} \right)\times \left(\mathcal{N} \cup \mathcal{L} \cup \{\mathbf{t}\}\right)$ captures relation between nodes.
    We define the sets $\pa n = \{n' \mid (n', n) \in \mathcal{E}\}$ and $\ch n = \{n' \mid (n, n') \in \mathcal{E}\}$ to represent the parent and children nodes of a given node $n$, respectively. 
\end{definition}
Note that in Fig.~\ref{fig: tree}, $\pa {\mathbf{s}} = \emptyset$, since sink has no parent node and $\pa 1 = \mathbf{s}$\footnote{We abuse notation and write $\pa{n}= n'$ instead of $\pa{n}= \{n'\}$ for readability since a node has at most one parent.}.
On the other hand, $\ch {\mathbf{t}}=\emptyset$ and for all $n\in \mathcal{L}$ we have $\ch n= \{\mathbf{t}\}$.
An internal node $n \in \mathcal{N}$ has three children, a left child, a right child, and the sink $\mathbf{t}$.
In the classification process using $\mathcal{T}$, data samples that are correctly classified traverse from the source node through the tree to reach the sink. 
Conversely, misclassified samples are blocked from progressing beyond the source node.
The binary decision tree (BDT) with STL primitive nodes is extracted from the classification tree $\mathcal{T}$.
The root of the STL BDT is the node connected to the source $\mathbf{s}$, i.e., node $1$ in Fig.~\ref{fig: tree}.
Not all nodes $\mathcal{N} \cup \mathcal{L}$ of $\mathcal{T}$ are part of the inferred STL BDT.
The final structure follows from the optimization result. 
Note that a pre-determined depth value limits the size of the classification tree.


\begin{figure}[hbt!]
    \centering
\includegraphics[width=0.5\linewidth]{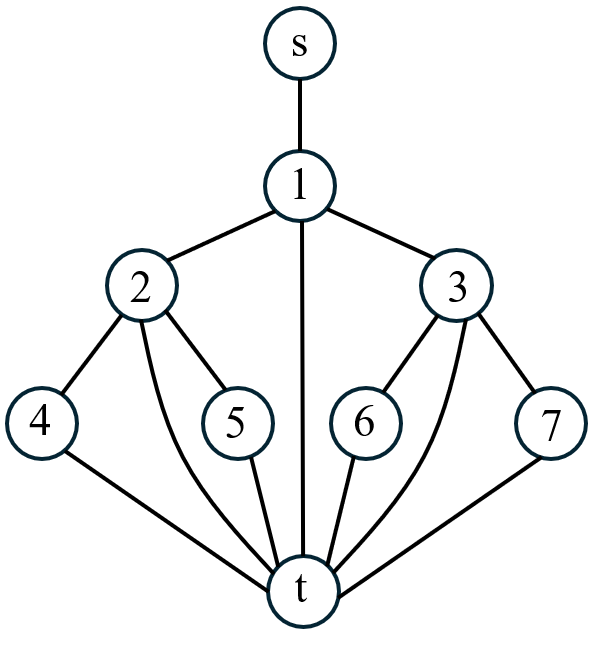}
    \caption{Max flow classification tree.}
    \label{fig: tree}
\end{figure}

%
%

To partition the data at each node $n\in \mathcal{N}$, a finite list of possible splitting rules is considered~\cite{ripley2007pattern}. 
We use simple Parametric Signal Temporal Logic (PSTL) formulae~\cite{asarin2012parametric}, called primitives~\cite{bombara2016decision}, 
defined as follows.
\begin{definition}[PSTL primitives]
\label{def:primitives}
The primitive in the set $\Psi$ consists of a temporal operator, a relational operator, and a predicate function.
Formally, 
$\Psi = \{\Gamma (h(s) \sim \pi) \mid
\Gamma \in \{\lalways_{I_1}, \levent_{I_1}, \lalways_{I_1}\levent_{I_2}, \cdots\},
\sim \in \{\geq , <\}, h \in \mathcal{H}\}$,
where $\pi \in \mathbb{R}$ is the threshold,
$I_j$ are integer ranges of temporal operators,
and $\mathcal{H}$ is a finite set of predicate functions
$h:\range{0}{H} \to \mathbb{M}$.
The set of all possible time intervals of temporal operators is
$\Theta=\{\theta = (I_1,\ldots,I_q) \mid I_p = \range{\underline{I}_p}{\bar{I}_p} \mid \underline{I}_p \leq \bar{I}_p, \; \text{with } \underline{I},\,\bar{I} \in \range{0}{H}, \forall p \in \range{1}{q}, q\in \mathbb{Z}_{\geq 1}\}$,
where $q$ is the number of temporal operators in a primitive.
\end{definition}

In Def.~\ref{def:primitives}, $\pi$ is the \emph{spatial parameter},
while the time bounds of temporal operators $I_p$ are the \emph{time parameters}.

When using a tree $\mathcal{T}$ for classification, we assign a primitive $\psi \in \Psi$ or a class $c \in \Classes$ to each node $n$ in $\mathcal{N}$ and $\mathcal{L}$. Note that a class $c \in \Classes$ can be assigned at internal or leaf nodes, whereas primitives can only be assigned to internal nodes.
For internal nodes $n \in \mathcal{N}$ in addition, we also compute the spatial and time parameters of primitive $\psi$, denoted as $\pi_n$ and $\theta_n = (I_1,\ldots,I_{q_\psi})$, where $q_\psi$ is the number of temporal operators in $\psi$.

For each data sample $s^i$ in the dataset $\Dataset$, the following encoding
captures the classification of $s^i$ as a flow from the source $\mathbf{s}$
to the sink $\mathbf{t}$.
At the starting node $\mathbf{s}$, $s^i$ flows down to the top node 1, see Fig.~\ref{fig: tree}.
If the node $n$ is a \emph{decision node} associate with the STL formula $\phi_n = \psi_n(\pi_n, \theta_{\psi_n})$,
then the check $s^i \models \phi_n$ is performed.
In case the signal $s^i$ satisfies $\phi_n$, the signal flows to the left child
$l(n) \in \ch{n}$. Otherwise, it flows to the right child $r(n) \in \ch{n}$.
The link to the sink node $\mathbf{t}$ is not used for decision nodes.
The procedure continues from the node reached by $s^i$.
If the node $n$ is a \emph{classification node} associated with class $c_n$,
the check $c_n = \ell^i$ is performed.
In case the label $\ell^i$ matches the class $c_n$, then the signal $s^i$ flows
to the sink $\mathbf{t}$ and is correctly classified and counted towards the objective.
If a signal cannot be correctly classified, it does not enter the classification tree and, thus, does not count toward the objective as explained later in Sec.~\ref{subsec: flow}.
In either case, the classification of $s^i$ ends when it reaches the sink $\mathbf{t}$.
For classification nodes, the links to the right and left nodes are unused.
Moreover, leaves $\mathcal{L}$ can only be classification nodes since
they have edges only to the sink $\mathbf{t}$.


Hence, to achieve accurate data classification, we aim to construct a classification tree $\mathcal{T}$ that maximizes the Correct Classification Rate ($CCR(\Phi)$). 
Consequently, the flow-based optimization problem solving Pb.~\ref{pb: p1} is:
\begin{equation}\label{eq: op}
\begin{aligned}
 \max & \quad CCR(\Phi) \\
  \text{s.t. }&\text{Flow constraints},\\
  & \text{Node function constraints},\\
  & \text{STL satisfaction constraints}.\\
\end{aligned}
\end{equation}

The objective is to maximize the Correct Classification Rate (CCR) while adhering to \emph{flow constraints} that ensure the conservation of data flow within the tree.
Constraints related to \emph{node functions} define primitive predicates, child allocation, and classification. 
The final soundness constraint ensures that the predicted class satisfies the signal. 
The following sections provide a detailed description of these constraints.

\subsection{Flow constraints}
\label{subsec: flow}

Here, we establish the conservation of data flow within the classification tree $\mathcal{T}$.
Let $ z^i_n \in \mathbb{B}$ denote a binary variable representing whether data point $(s^i,\ell^i) \in \Dataset$ traverses node $n \in \mathcal{N}\cup \mathcal{L}$. 
It takes a value of one if data point $(s^i,\ell^i)$ flows through node $n$ and zero otherwise.
We designate the entry of a data point into the source node as $z_\mathbf{s}^i \in \mathbb{B}$. 
Since all internal nodes are connected to a common sink node $\mathbf{t}$, we introduce $z_{\mathbf{t},n}^i \in \mathbb{B}$ which also identifies the node through which data point $(s^i,\ell^i)$ enters the sink node $\mathbf{t}$. 
Specifically, $z_{t,n}^i = 1$ indicates that data $(s^i,\ell^i)$ reaches the sink node via node $n$, and zero otherwise.
Then, the flow constraints are defined as follows
\begin{subequations}\label{eq: flow_constr}
\begin{align}
z_\mathbf{s}^i &= z_1^i, \quad \forall i\in\mathcal{I}\label{eq: flow0},\\
z_n^i &= z_{\mathbf{t}, n}^i + z_{l(n)}^i + z_{r(n)}^i, \quad \forall n\in\mathcal{N}, \forall i\in\mathcal{I},\label{eq: flow1}\\
z_n^i &= z_{\mathbf{t},n}^i, \quad \forall n\in\mathcal{L}, \forall i\in\mathcal{I}.\label{eq: flow2}\\
\notag
\end{align}
\end{subequations}%
The constraint~\eqref{eq: flow0} enforces a flow from the source node $\mathbf{s}$
to enter node $n=1$ corresponding to the root of the inferred BDT that eventually reaches sink $\mathbf{t}$ and is correctly classified.
In this case, $z_\mathbf{s}^i = z_1^i =1$.
Otherwise, we have $z_\mathbf{s}^i = z_1^i = 0$, and there is no flow for $s^i$, and it is thus misclassified.
Next, \eqref{eq: flow1} ensures that data $s^i$ leaving internal node $n\in \mathcal{N} \cup \mathcal{L}$ must enter one of its children nodes or the sink node $\mathbf{t}$.
The last equation \eqref{eq: flow2} enforces the flow for $s^i$ from leaves $\mathcal{L}$ to the sink $\mathbf{t}$, which is their only child.
These constraints imply that data is either correctly classified, i.e., $z_{\mathbf{t},n}^i = 1$ for a node $n \in \mathcal{N} \cup \mathcal{L}$ via a deterministic tree flow, or it does not enter the tree, i.e., $z_\mathbf{s}^i = 0$.

\subsection{Node function constraints}

In this section, we capture the functionality of the nodes in $\mathcal{T}$.
Internal nodes $\mathcal{N}$ are either decision or classification nodes.
Each internal node $n \in \mathcal{N}$ in $\mathcal{T}$ either checks a primitive STL formula and creates child nodes or classifies data by directing the flow to the sink.
Leaf nodes $n\in \mathcal{L}$ exclusively make classifications by enforcing the flow to the sink $\mathbf{t}$.
For each node $n\in \mathcal{N}$,
we introduce the binary decision variables $b_{n}^{\psi}\in\mathbb{B}$
to capture whether $n$ is a decision node associated with primitive STL function $\psi \in \Psi$.
Conversely, binary decision variables $ w_n^c \in \mathbb{B}$ indicate that
node $n$ is a classification node for class $c \in \Classes$.


The constraints governing the node functionality are:
\begin{subequations}\label{eq: node_contr_1}
\begin{align}
\sum_{\psi\in\Psi}b^{\psi}_n+ \sum_{c\in\mathcal{C}} w_n^c &= 1, \quad \forall n\in\mathcal{N}, \label{eq: f_internal}\\
\sum_{c\in\mathcal{C}} w_n^c &= 1, \quad \forall n \in \mathcal{L}. \label{eq: f_leaf}\\
\notag
\end{align}
\end{subequations}
For internal nodes $n\in \mathcal{N}$, \eqref{eq: f_internal} guarantees
that $n$ either performs a decision using a primitive $\psi$ or classifies the data into a class $c$.
For leaves $n\in \mathcal{L}$, \eqref{eq: f_leaf} enforces classification, i.e., assignment of a class $c$.

When imposing a decision using a primitive, time intervals and thresholds are required to complete the STL formula. 
The time parameters variable $\xi_{\theta}^{n} \in\mathbb{B}$, $\forall \theta\in \Theta$,
encompasses all possible valuations of the time parameters within the horizon $H$. 
The constraint capturing the primitive variable is 
\begin{equation}\label{eq: time window}
\sum_{\psi\in\Psi}b^{\psi}_n = \sum_{\theta\in\Theta}\xi_{\theta}^{n}, \quad \forall n\in\mathcal{N}.
\end{equation}
If a primitive is selected, i.e., $\sum_{\psi\in\Psi}b^{\psi}_n=1$,
then \eqref{eq: time window} ensures that exactly one set of time parameters $\theta$ is selected for the primitive $\psi$.
Otherwise, no $\theta$ is selected.

For classification nodes responsible, data flows to the sink node $\mathbf{t}$ if the predicted class is correct.
The following constraint enforces this behavior. The inequality constraint is enforced since multiple classification nodes can predict the same class and the data can only flow into one of the nodes.
\begin{equation}\label{eq: node_contr_2}
z_{\mathbf{t},n}^i \leq w_n^c, \quad \forall i: \ell^i =  c \in \mathcal{C}, \, n\in \mathcal{N}\cup\mathcal{L}.
\end{equation}

\subsection{STL satisfaction constraints}
\label{subsec:satisfaction_constr}
In this section, we connect the decision variables for decision nodes with the satisfaction of primitive formulas.
Additionally, the constraints capture the robustness of signals with respect to the STL formula obtained from a PSTL primitive and computed spatial and time parameters.


\subsubsection{Robustness calculation}
To capture the robustness of signals with respect to primitives, we use the following result.

\begin{proposition}
\label{thm:primitive-robustness}
Let $s$ be a signal and $\psi \in \Psi$ be a PSTL primitive.
For any valuations $\theta$ and $\pi$ of the spatial and time parameters of $\psi$, 
we have
\begin{equation*}
    \rho(s, \psi(\pi, \theta)) = \rho(s, \psi(0, \theta)) - \pi
\end{equation*}
\end{proposition}
\begin{proof}
The proof follows trivially from the structure of primitive formulae in $\Phi$,
and is omitted for brevity.
\end{proof}

Using Prop.~\ref{thm:primitive-robustness}, we can precompute the robustness values
for all signals $s^i$, primitives $\psi$ and time parameters $\theta$,
which we denote as $\tau_{\theta}^{i, \psi}$.
For example, the robustness of signal $s^1$ with respect to
$\psi_1 = \levent_{\range{1}{5}} h(s) \geq 0$
is $\tau_{\range{1}{5}}^{1, \psi_1} = \max_{k\in \range{1}{5}} h(s^1(k))$.

The robustness $\rho_\psi^{i, n} \in \mathbb{R}$ of data $s^i$ with respect to the formula $\psi$
with an arbitrary spatial parameter $\pi_n \in\mathbb{R}$ is captured by
\begin{equation}
\label{eq: true_rho}
\sum_{\theta\in\Theta}\xi_{\theta}^{n} \tau_{\theta}^{i, \psi} + \pi_n = \rho^{i, n}_\psi, \quad \forall i\in\mathcal{I}, \,\psi\in\Psi, \, n\in \mathcal{N}.
\end{equation}
We enforce constraint~\eqref{eq: true_rho} for all internal nodes $n\in\mathcal{N}$
to compute the robustness of every signal with respect to the predicted formula at node $n$.

\subsubsection{Soundness constraints}

If a signal satisfies the primitive formula at a decision node $n\in\mathcal{N}$, it is directed to the left child node $l(n)$.
Otherwise, it proceeds to the right child node $r(n)$.

To represent the sign of robustness, we introduce the binary variables $y^{i, n}_{\psi} \in \mathbb{B}$. 
The variable $y^{i, n}_{\psi}$ takes the value of one if $\rho^i_{\phi} \geq 0$, and zero if $\rho^i_{\phi} \leq 0$. 
We employ the \emph{big-M} method to encode this relationship such that the robustness value is not over-constrained when $y^i_{\phi}=0$ as follows
\begin{equation}
\label{eq: rho2binary}
\begin{aligned}
     \rho^{i, n}_{\psi}  &\leq M y^{i, n}_{\psi}, \quad \forall i\in\mathcal{I}, \psi\in\Psi,n\in\mathcal{N},\\
     -\rho^{i, n}_{\psi}  &\leq M (1 - y^{i, n}_{\psi}), \quad \forall i\in\mathcal{I}, \psi\in\Psi,n\in\mathcal{N}.
\end{aligned}
\end{equation}
where $M = \rho_\top \in \mathbb{R}_{>0}$ is the constant value of the maximum robustness.



Next, we need to connect the satisfaction of primitives at a decision node $n$
with the flow to the left or right child.
Hence, for indicating data flowing to the left child node, we introduce variable $z_{l(n)}^i \in \mathbb{B}$ which takes the value of one if two conditions are satisfied simultaneously 
(1) $n$ is a decision node with primitive $\psi$, i.e., $b^{\psi}_n = 1$, and 
(2) the robustness is positive with respect to the formula $y^{i, n}_{\psi} = 1$.
The conditions are enforced by
\begin{equation}\label{eq: min_left}
    z^i_{l(n)}\leq \sum_{\psi\in\Psi} \kappa_{\psi}^{i,n},\quad
     \forall i\in\mathcal{I}, \forall n\in\mathcal{N},
\end{equation}
where $\kappa_{\psi}^{i,n}= y^{i,n}_{\psi}\cdot b^{\psi}_{n}$ is an auxiliary variable that captures the condition
$z_{l(n)}^i=1$ if and only if $y^{i, n}_{\psi}=1$ and $b^{\psi}_{n}=1$. 
Note that since both variables are binary, the product is equivalent to $\kappa_{\psi}^{i,n}= \min\{y^{i,n}_{\psi},b^{\psi}_{n}\}$
and is encoded as mixed integer linear constraints.

Data flowing to the right $z^i_{r(n)}\in\mathbb{B}$ follows a similar process but with opposite conditions as follows:
\begin{equation}\label{eq: min_right}
\begin{aligned}
z^i_{r(n)}&\leq 1-\sum_{\psi\in\Psi}\kappa_{\psi}^{i,n},\quad
     \forall i\in\mathcal{I},  n\in\mathcal{N}.\\
        z^i_{r(n)} &\leq \sum_{\psi\in\Psi}b^{\psi}_{n}.\\
\end{aligned}
\end{equation}

\subsection{MILP formulation}

The $CCR(\Phi)$ of the data set, which is our main objective in the optimization problem, is given by the flow into sink $\mathbf{t}$.
Thus, we obtain the objective by summing up $z_{\mathbf{t},n}^i$ as follows:
\begin{equation}
\label{eq: o_obj}
CCR(\phi)=\frac{\sum_{i\in\mathcal{I}}\sum_{n\in\mathcal{N}\cup\mathcal{L}}z_{\mathbf{t},n}^i}{|\mathcal{I}|}.
\end{equation}
Moreover, the optimal flow tree structure enables specifying a penalty
for the number of decision nodes, thereby reducing the complexity of the resulting formula. 
However, this approach may involve a trade-off in performance, as more decisions can lead to better classification, i.e., larger CCR. 
The trade-off is captured by a variable $\lambda \in [0,1]$ that serves as a regularization parameter.
This parameter allows users to adjust the emphasis between complexity reduction and maintaining high performance. 
Consequently, the objective function in \eqref{eq: o_obj} is modified as follows:
\begin{equation}
\label{eq: n_obj}
    \mathcal{J}=(1-\lambda)\sum_{i\in\mathcal{I}}\sum_{n\mathcal{N}\cup\mathcal{L}}z_{\mathbf{t},n}^i + \lambda \sum_{\psi\in\Psi}b^{\psi}_n.
\end{equation}
The first part of the objective function evaluates the number of correctly classified data, while the second part quantifies the number of decision nodes with $\lambda$ serving as the adjustment variable.

Finally, problem in~\eqref{eq: op} becomes the MILP problem:
\begin{equation}\label{eq: op_pb}
\begin{aligned}
 \text{max }& \mathcal{J}\\
  \text{s.t. }&\eqref{eq: flow_constr} \quad \text{(Flow constraints),}\\
  & \eqref{eq: node_contr_1}-\eqref{eq: node_contr_2} \quad \text{(Node function constraints),}\\
  & \eqref{eq: true_rho}-\eqref{eq: min_right} \quad \text{(STL satisfaction constraints)}.\\
\end{aligned}
\end{equation}

\subsection{Problem simplification}


In this section, we present a simplified version of the problem in~\eqref{eq: op_pb}
when considering first-level STL primitive formulae, i.e., formulae without nested
temporal operators~\cite{bombara2016decision}. 


Under these two conditions, the following primitives are equivalent:
\begin{equation}\label{eq: equiv}
    \begin{aligned}
        \rho(s, \levent_{\theta} h(s) &\geq \pi) = - \rho(s, \lalways_{\theta} h(s) \leq \pi) \\
        \rho(s, \levent_{\theta} h(s) &\leq  \pi) = -\rho(s, \lalways_{\theta} h(s) \geq \pi)
    \end{aligned},
\end{equation}

\begin{proof}
From the definition, we have
    \begin{equation}
        \begin{aligned}
\rho(\levent_{\theta} h(s)\geq \pi) &= \max_{t \in \theta} \left\{h(s(t)) - \pi \right\}\\
&=- (- \max_{t \in\theta} \left\{h(s(t)) - \pi\right\})\\
&=- (\min_{t\in\theta} \left\{\pi - h(s(t)) \right\})\\
&= -\rho(\lalways_{\theta}h(s) \leq \pi)
        \end{aligned}.
    \end{equation}
The proof for the second equivalence follows similarly.
\end{proof}

Inferring a primitive $\lalways_{\theta} h(s) \leq \pi$ from data is the
same as learning $\levent_{\theta} h(s) \geq \pi$.
We remove the primitives containing the always operator and, thus,
reduce $\Psi$ by half.
The simplification can be employed with any set $\Phi$ closed under negation.
For example, second level primitives of the form $\lalways_{\theta_1} \levent_{\theta_2}$ and $\levent_{\theta_1} \lalways_{\theta_2}$.


\subsection{STL formula summary}
Once the optimization in \eqref{eq: op_pb} is completed, we can trace all STL formulae starting from the first layer for each data class. 
Beginning with the first layer, the traversal to a left child from any node signifies the satisfaction of the STL formula associated with that node, whereas progressing to a right child indicates the negation of the STL formula ($\neg$). 
As we proceed through successive layers, the STL formulae encountered along the path are unified using the logical conjunction ($\land$). 
Upon reaching a class prediction node, the aggregated conjunction formula essentially defines the class for that node's prediction. 
Lastly, in cases where a class is represented in multiple nodes across the tree, the distinct STL formulae corresponding to each node are combined using the logical disjunction ($\lor$).

\section{Case Studies}
\label{sec: res}
We demonstrate the effectiveness of our method in three case studies. 
The first is the two-class naval surveillance dataset with comparisons to other binary STL classification methods. 
The second case study demonstrates the application of the proposed approach to multi-class STL inference using a four-class trace dataset. 
Finally, the third case study is a designed dataset 
that uses second-level STL primitives for classification.  
The adjustment variable $\lambda$ is set to 0 to maximize $CCR$.
All computations of the following case studies were performed on a PC with 20 cores at 3.7 GHz with 32 GB of RAM.
We use Gurobi~\cite{gurobi} as the MILP solver. 

\subsection{Naval Surveillance}
In this case study, we use the naval surveillance dataset from \cite{kong2016temporal}. 
The different classes represent two behaviors of the vessel trajectories, as shown in Fig.~\ref{fig:ns}. 
All trajectories start from the open sea on the right side of the figure. 
The blue trajectories are normal behaviors that directly head toward the harbor. 
The green and red trajectories are abnormal. 
The red trajectories approach the island and return to the open sea, while the green trajectories veer to the island and then head to the harbor. 
All trajectories contain 61 time steps of $x$ and $y$ positions.
\begin{figure}[hbt!]
    \centering
    \includegraphics[width=1\linewidth]{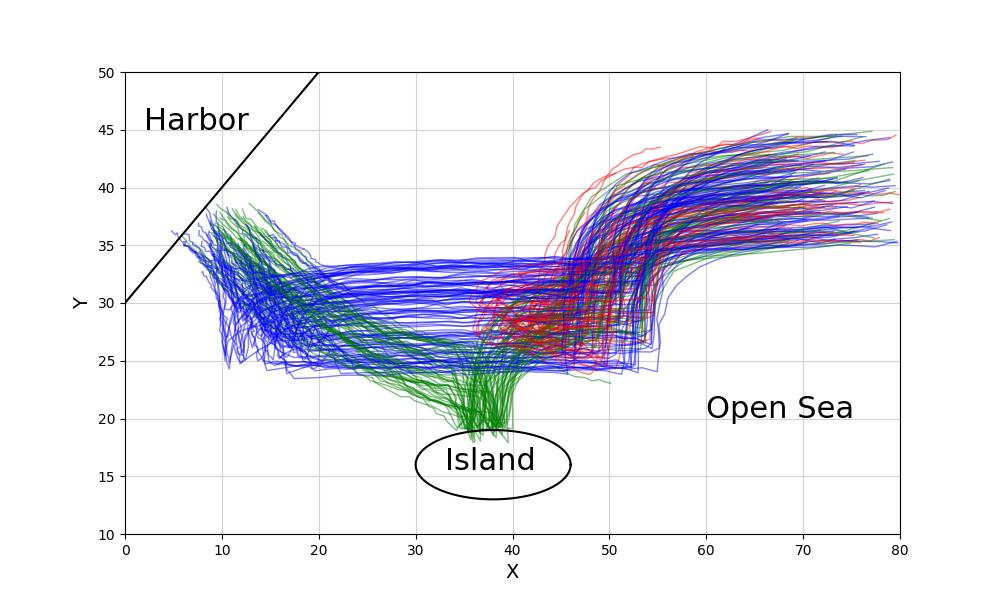}
    \caption{Naval Surveillance Dataset}
    \label{fig:ns}
\end{figure}

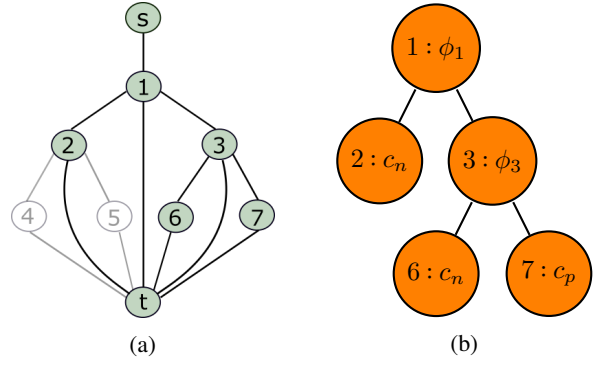
\begin{figure}
\centering
\begin{minipage}{.54\linewidth}
  \centering
  \includegraphics[width=.75\linewidth, clip]{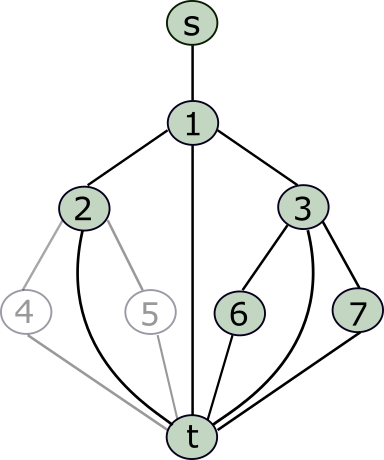}
  \subcaption{}
  \label{subfig:classification_tree}
\end{minipage}%
\begin{minipage}{.45\linewidth}
  \centering
\begin{tikzpicture}[nodes={draw, circle, fill=orange}, -, shorten <=.5pt, thick][scale=.3]
\node{$1:\phi_1$}
            child {node[name=2] {$2:c_n$} 
                    }
            child {node {$3:\phi_3$}
                    child {node[name=6] {$6:c_n$} }
                    child {node[name=7] {$7:c_p$} 
                           }};

\end{tikzpicture}
  \subcaption{}
  \label{subfig:dt}
\end{minipage}
\caption{\ref{subfig:classification_tree} depicts the classification tree of depth 2 wherein only green-colored nodes are used for classification while grey nodes and edges are redundant. Fig.~\ref{subfig:dt} shows the resulting STL BDT wherein $c_p$ denotes ''normal" and $c_n$ indicates ''abnormal" behavior. }
\end{figure}

We use the optimal flow tree with a depth of 2 for this dataset and compare other methods from \cite{aasi2022classification}. 
All results are computed with an average of 10 runs. 

The classification tree and the resulting decision tree are shown in Fig.~\ref{subfig:classification_tree} and Fig.~\ref{subfig:dt}, respectively. 
The colored nodes indicate nodes utilized for classification. 
A typical example of the learned formula from our method for the normal and abnormal trajectories is:
\begin{multline*}
    \begin{aligned}
        &\text{node1: }\phi_1: \levent_{[13,59]} y \leq 23.168 \\
        &\text{node2: abnormal}\\
        &\text{node3: }\phi_3:  \levent_{[50,59]} x \geq 25.483 \\
        &\text{node6: abnormal}\qquad \text{node7: normal}\\
                &\text{STL formula for normal behavior $c_p$: }\neg \phi_1 \land \neg\phi_3 \\
    &\text{STL formula for abnormal behavior $c_n$: }\phi_1\lor \phi_3
    \end{aligned}
\end{multline*}
The interpretation of this formula is clear: the $y$ coordinates of normal vessels do not reach the island between 13 and 59 seconds, and the $x$ coordinates of normal vessels eventually enter the harbor between 50 and 59 seconds, while abnormal trajectories are the opposite. 

We compare the results of \cite{aasi2022classification} using boosted concise decision trees (BCDTs) with respect to the average runtime and the classification accuracy in Fig.~\ref{fig: diff pare} by varying the number of samples. 
In the BCDTs approach, the variable $k$ is the different number of decision trees used for classification, and $d$ is the depth of the tree. 
The results demonstrate that our approach consistently outperforms the BCDTs approach in obtaining optimal solutions.

\begin{figure}
    \centering
    \subfloat[\label{subfig: runtime}]{%
    \includegraphics[width=0.8\linewidth]{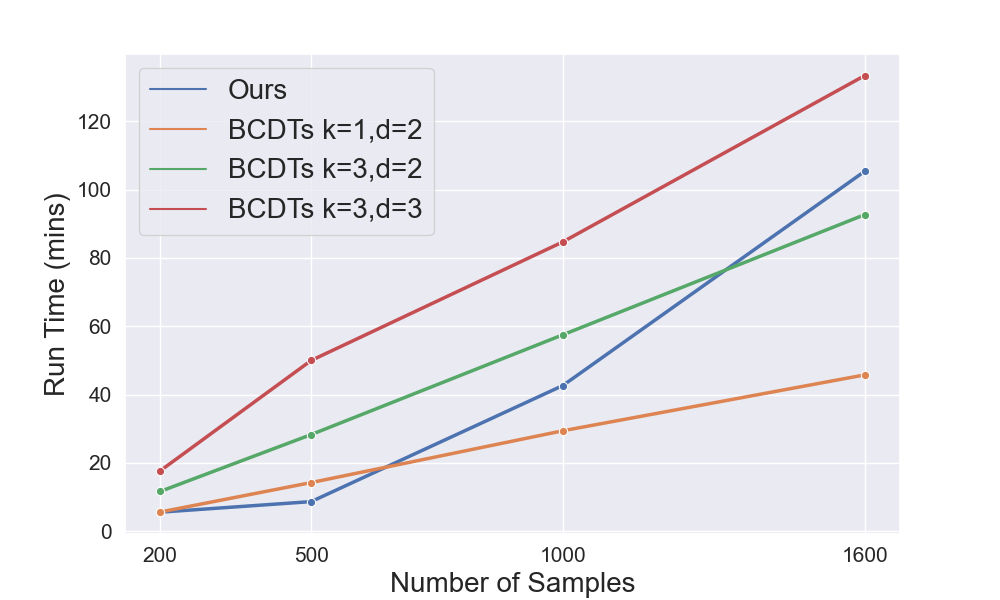}}
    \vspace{-0.01in}
    \centering
    \subfloat[\label{subfig: ccr}]{%
       \includegraphics[width=0.8\linewidth]{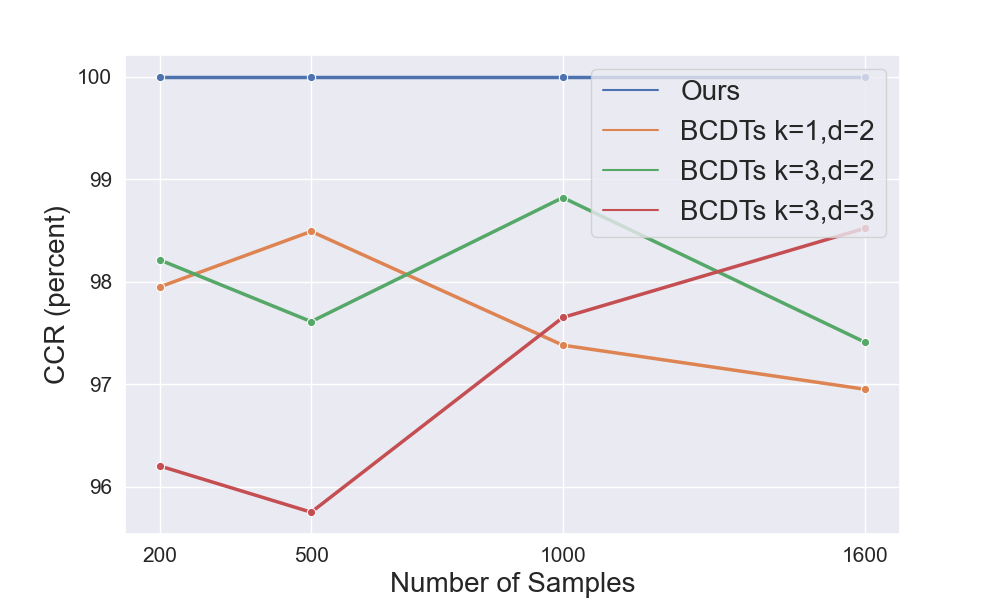}
       }
    \caption{Performance comparison (a) Run time (b) Correct classification rate.}
    \label{fig: diff pare}
\end{figure}

\subsection{Trace Dataset}
The trace dataset contains four classes of one-dimensional transient behavior from a simulated nuclear industry process~\cite{dau2019ucr}. 
The length of the data is 275, a total of 200 samples in the dataset, evenly distributed by classes as shown in Fig.~\ref{fig:trace}. 
\begin{figure}
    \centering
    \includegraphics[width=1\linewidth]{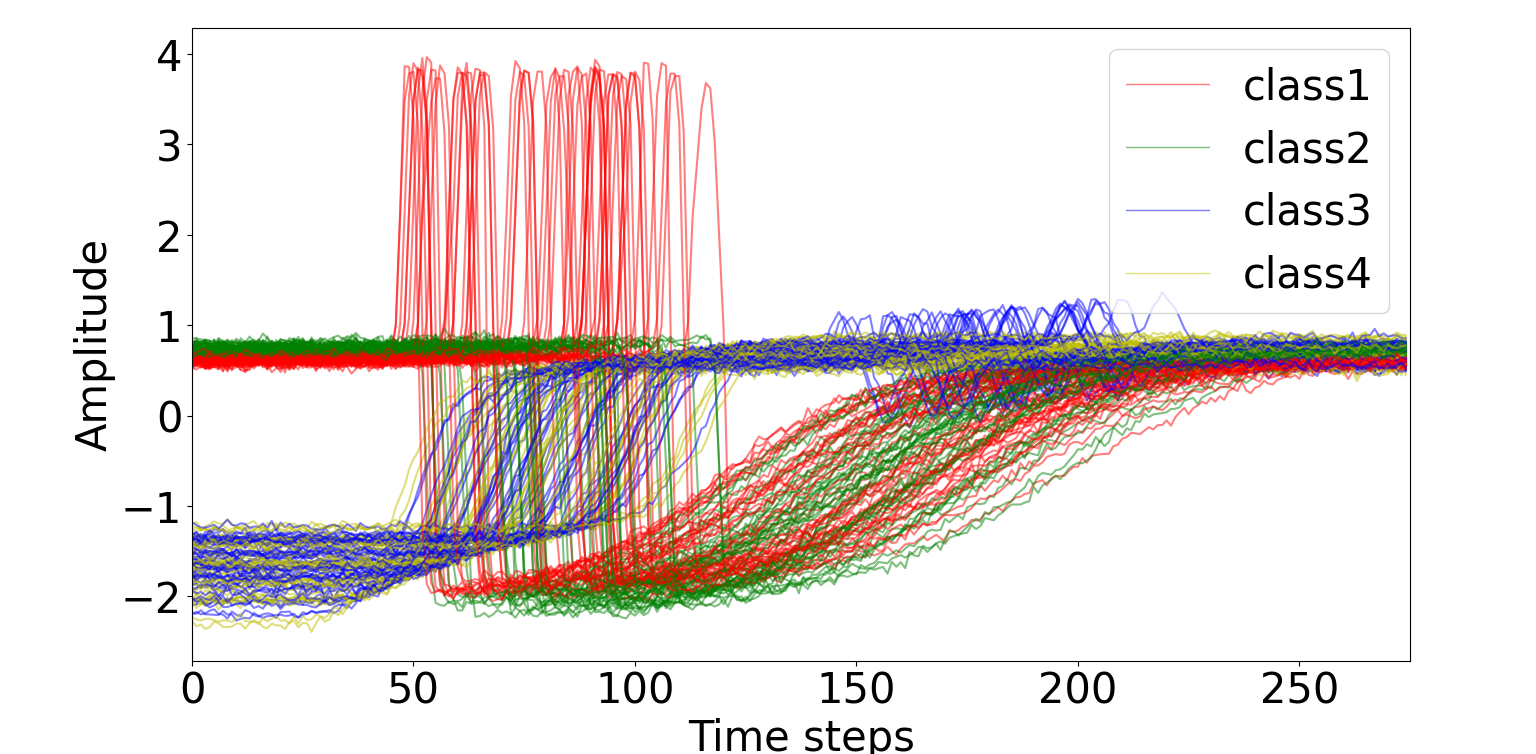}
    \caption{Trace Dataset}
    \label{fig:trace}
\end{figure}
We used a tree with a depth of 3 for this dataset, and the objective achieves 100 percent accuracy and the resulting formulae for the four classes are computed as follows:
\begin{multline*}
\begin{aligned}
&\text{node1:} \phi_1: \levent_{[173,181]} s\leq 0.74 \quad \text{node2:} \phi_2: \levent_{[12,16]} s\geq 0\\
&\text{node3:} \phi_3: \levent_{[115,167]}s\leq 0.59\quad\text{node4:} \phi_4: \lalways_{[21,27]}s<0.72\\
&\text{node5:} \phi_5: \levent_{[157,216]}s\leq 0.35\quad\text{node6:} \phi_6: \levent_{[230,234]}s\geq0.88\\
&\text{node7:} \phi_7:  \lalways_{[210,210]}s<0.69\quad\text{node8: class 2} \\
&\text{node 9: class 1} \quad \text{node 10 \&13 \&15: class 3}\\
&\text{node 11 \&12 \&14: class 4}
\end{aligned}
\end{multline*}
\begin{multline*}
\begin{aligned}
    \text{class1: }&\phi_1\land \phi_2 \land \phi_4\\
    \text{class2: }&\phi_1\land\phi_2 \land \neg \phi_4\\
    \text{class3: }&
        (\phi_1\land \neg\phi_2 \land\phi_5)
        \lor (\neg\phi_1\land \phi_3 \land\neg\phi_6)\\
        \lor& (\neg\phi_1\land\neg\phi_3\land \phi_7)\\
    \text{class4: }&
        (\phi_1\land \neg\phi_2 \land\neg\phi_5)
        \lor (\neg\phi_1\land \phi_3 \land\phi_6)\\ \lor& (\neg\phi_1\land\neg\phi_3\land \neg\phi_7)
\end{aligned}
\end{multline*}
The results demonstrate the capability for multi-class classification.


\subsection{Classification using Second-level STL Primitives}
Our framework accommodates complex STL formulae. Fig.~\ref{fig:l2} is an illustrative example of sample trajectories designed to highlight the expressivity of second-level STL formula.
Current neural network-based approaches are limited to first-level formulae, and adapting them to other primitive sets is non-trivial~\cite{li2023learning}.

Class 1 of the blue trajectory is a triangle wave, and we complement the data with variations of plateaued triangular waves, such as the yellow and green trajectories with plateaus of 2 and 3, respectively. 
Additionally, to ensure that the inferred formula is not trivially defined by one tree level, we have trajectories of constant values at the lower and upper bounds drawn in purple.
Next, we sample different trajectories from each type of wave by shifting the initiation point.
We generate 300 samples with 15 time steps, 100 of which are class 1, and we denote the rest of the data as class 2. 
Lastly, we add a small random noise between -0.01 to 0.01 to all values.

We use a tree with a depth of 2 for classification. The optimization returns 100 percent accuracy, and the results are the following:
\begin{multline*}
    \begin{aligned}
        &\text{node 1: }\phi_1: \lalways_{[2, 11]}\levent_{[0,4]} s \leq 1.0086 \\
        &\text{node 2: class 2}\\
        &\text{node 3: }\phi_3: \lalways_{[0, 10]}\levent_{[0,2]}s \leq 3.017 \\
        &\text{node 6: class 1}\qquad \text{node 7: class 2}\\
    \end{aligned}
\end{multline*}




$\phi_1$ means within the time interval $[2,11]$, there is always a time within any subinterval of length 4 where the signal value is less than or equal to 1.0086. 
Therefore, only the lower constant trajectory satisfies $\phi_1$. 
Next, $\phi_3$ means that within the time interval $[0, 10]$, there is always an event within any subinterval of length 2 where the signal value is less than or equal to 3.017. 
From the remaining data that does not satisfy $\phi_1$, only the class 1 triangular wave satisfies $\phi_3$. As a result, the STL formula for class 1 is $\neg \phi_1 \land \phi_3$.

\begin{figure}
    \centering
    \includegraphics[width=1\linewidth]{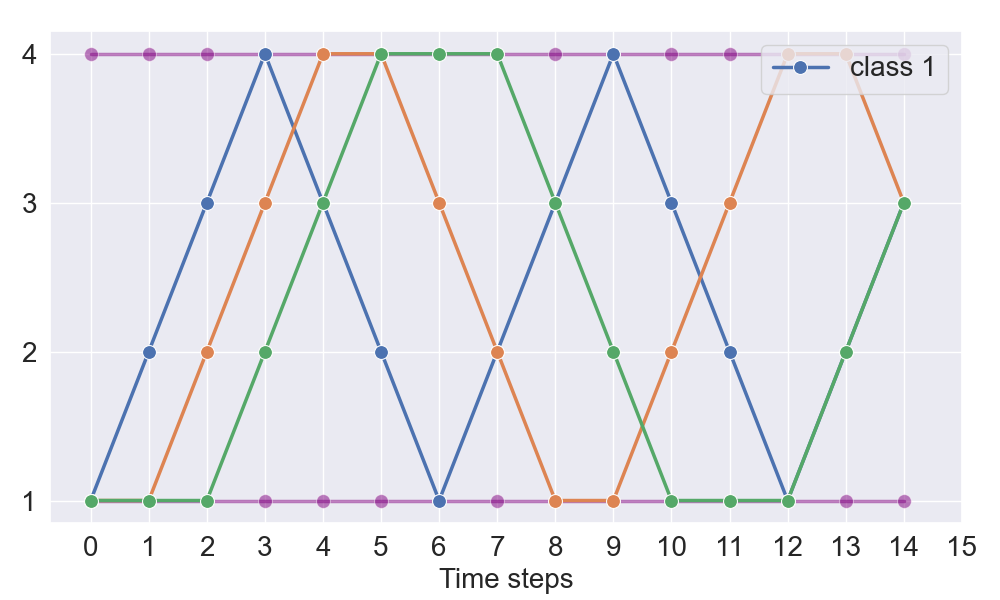}
    \caption{Trajectories for high-level STL}
    \label{fig:l2}
\end{figure}

\section{Conclusions}
In this paper, we propose a decision tree-based STL inference algorithm that works for both two-class and multi-class classification problems. 
Our algorithm formulates the problem as a MILP by leveraging a max-flow approach over a synthesized tree using STL primitives, which allows for the computation of a globally optimal solution, i.e., a high Correct Classification Rate (CCR).
Additionally, we exploit the symmetry of STL primitives to reduce the required constraints for classification, resulting in improved performance and reduced complexity.
We evaluate the performance through three case studies to demonstrate the ability to interpret complex STL formulas and the precision in predictive accuracy.
We show the capacity of our approach to handling two-class and multi-class classification problems with first-level STL primitives and a case for second-level STL primitives.


\bibliographystyle{ieeetr}
\bibliography{references.bib}
\end{document}